\documentclass[11pt,letterpaper]{article}
\usepackage{amssymb,amsmath,amsthm,enumerate,nicefrac}
\usepackage{mathptmx}
\usepackage{graphicx, color}
\usepackage[english]{babel}
\usepackage[driverfallback=hypertex,pagebackref=false,colorlinks]{hyperref}
\hypersetup{linkcolor=[rgb]{.7,0,0}}
\hypersetup{citecolor=[rgb]{0,.7,0}}
\hypersetup{urlcolor=[rgb]{.7,0,.7}}
\usepackage{geometry}

\usepackage{bbm}
\usepackage{authblk}
\usepackage{epstopdf}
\AtBeginDocument{%
  \addtolength\abovedisplayskip{-0.25\baselineskip}%
  \addtolength\belowdisplayskip{-0.15\baselineskip}%
  \addtolength\abovedisplayshortskip{-0.5\baselineskip}%
  \addtolength\belowdisplayshortskip{-0.5\baselineskip}%
}

\newtheorem{theorem}{Theorem}

\newtheorem{claim}{Claim}

\newtheorem{proposition}{Proposition}

\newtheorem*{theorem*}{Theorem}

\newcommand{\R}{\mathbb{R}}

\newcommand{\X}{\mathcal{X}}

\newcommand{\N}{\mathcal{N}}

\newcommand{\Lp}[1]{\text{Lip}_1(#1)}
\newcommand{\one}{\mathbbm{1}}
\newcommand{\burden}{\text{cost of strategy}}

\title{The Role of Randomness and Noise in Strategic Classification}
\author[1]{
  Mark Braverman \thanks{
    Email: {\tt  mbraverm@cs.princeton.edu}. Research supported in part by the NSF Alan T. Waterman Award, Grant No. 1933331, a Packard Fellowship in Science and Engineering, and the Simons Collaboration on Algorithms and Geometry. Any opinions, findings, and conclusions or recommendations expressed in this publication are those of the author and do not necessarily reflect the views of the National Science Foundation. 
  }}
      \author[1]{
  Sumegha Garg\thanks{
    Email: {\tt sumegha.garg@gmail.com}.
  }
 }
 \affil[1]{ Department of Computer Science,
  Princeton University}
\date{}
\begin{document}
\maketitle
\sloppy
\begin{abstract}
We investigate the problem of designing optimal classifiers
in the ``strategic classification" setting, where the classification is part of a game in which players can modify their features to attain a favorable classification outcome (while incurring some cost). Previously, the problem has been considered from a learning-theoretic perspective and from the algorithmic fairness perspective. 

Our main contributions include
\begin{itemize}
	\item  Showing that if the objective is to maximize the efficiency of the classification process (defined as the accuracy of the outcome minus the sunk cost of the qualified players manipulating their features to gain a better outcome), then using randomized classifiers (that is, ones where the probability of a given feature vector to be accepted by the classifier is strictly between $0$ and $1$) is necessary.
	\item 
	Showing that in many natural cases, the imposed optimal solution (in terms of efficiency) has the structure where players never change their feature vectors (and the randomized classifier is structured in a way, such that the gain in the probability of being classified as a `1' does not justify the expense of changing one's features).
	\item 
	Observing that the randomized classification is not a \emph{stable} best-response from the classifier's viewpoint, and that the 
	classifier doesn’t benefit from randomized classifiers without creating instability in the system. 
\item 
	Showing that in some cases, a \emph{noisier signal} leads to better equilibria outcomes --- improving both accuracy and fairness when more than one subpopulation with different feature adjustment costs are involved. This is particularly interesting from a policy perspective, since it is hard to force institutions to stick to a particular randomized classification strategy (especially in a context of a market with multiple classifiers), but it is possible to alter the information environment to make the feature signals inherently noisier. 
\end{itemize} 
\end{abstract}
\clearpage
\section{Introduction}\label{sec:intro}

Machine learning algorithms are increasingly being used to make decisions about the individuals in various areas such as university admissions, employment, health, etc. As the individuals gain information about the algorithms being used, they have an incentive to adapt their data so as to be classified desirably. For example, if a student is aware that a university heavily weighs SAT score in their admission process, she will be motivated to achieve a higher SAT score either through extensive test preparation or multiple tries. Such efforts by the students might not change their probability of being successful at the university, but are enough to fool the admissions' process. Therefore, under such ``strategic manipulation" of one's data, the predictive power of the decisions are bound to decrease. One way to prevent such manipulation is by keeping the classification algorithms a secret, but this is not a practical solution to the problem, as some information is bound to leak over time and the transparency of these algorithms is a growing social concern. Thus, this motivates the study of algorithms that are optimal under ``strategic manipulation". The problem of gaming in the context of classification algorithms is a well known problem and is increasingly gaining researchers' attention, for example, \cite{hardt2016strategic,akyol2016price, hu2018disparate,milli2018social,dong2018strategic}. 

\cite{bruckner2011stackelberg} and \cite{hardt2016strategic} modeled strategic classification as a Stackelberg competition-- the algorithm (Jury) goes first and publishes the classifier, and then the individuals get to transform their data, after knowing the classifier, incurring certain costs to manipulate. The individuals would manipulate their features as long as the cost to manipulate is less than the advantage gained in getting the desirable classification. We assume that such manipulations don't change the actual qualifications of an individual. A natural question is: what classifier achieves optimal classification accuracy under the Stackelberg competition? These papers considered the task of strategic classification when the published classifier is deterministic. We study the role of randomness (and addition of noise to the features) in strategic classification and define the Stackelberg equilibrium for probabilistic classifiers, that assigns a real number in $[0,1]$, to each individual and a classification outcome $o$, representing the probability of being classified as $o$. 

As higher SAT scores are preferred by a university, the students would put an effort in increasing their SAT score, thereby, forcing the university to raise the score bar to optimize its accuracy (under the Stackelberg equilibrium). Due to this increased bar of acceptance, even the students who were above the true cutoff would have to put an extra effort to achieve a SAT score above this raised bar. And this effort is entirely the result of gaming in the classification system. We define the \emph{$\burden$}  for a published classifier to be the total extra effort, it induced, amongst the qualified individuals of the population. Then, we define the \emph{efficiency} of a published classifier to be its classification accuracy minus the $\burden$  under the Stackelberg equilibrium. A natural question here is: what classifier achieves the optimal efficiency? The efficiency of a published classifier represents the total impact of the classifier on all the agents in the Stackelberg equilibrium.

In normal classification problems it is never a good idea to use randomness, since one should always adhere to the best/utility maximizing action based on the prediction. Just as in games, randomness may lead to better solution in strategic classification., the paper aims to start understanding tradeoffs between efficiency losses due to randomness and efficiency gains through better equilibria induced by the randomized classifier.

Gaming in classification adds to the plethora of fairness concerns associated with classification algorithms, when the costs of manipulation are different across subpopulations. For example, a high weightage of SAT scores (for university admissions) favors the subgroups of the society that have the resources to enroll in test preparation or attempt the test multiple times. Further, varying costs across the subpopulations can lead to varied efforts put by identically qualified individuals, belonging to different subpopulations, to achieve the same outcome. \cite{milli2018social} and \cite{hu2018disparate} study the disparate effects of strategic classification on subpopulations (we will discuss these papers more in the related work section). \cite{hu2018disparate} observes that a single classifier might have different classification errors on subpopulations due to the varying cost of manipulations. We also study the effect of strategic manipulation on the classification errors across subpopulations and how randomized classifiers or noisy features may reduce the disparate effects. 

Strategic classification is a well known problem and there has been research in many other aspects of strategic classification, for example, learning the optimal classifier efficiently when the samples might also be strategic \cite{hardt2016strategic, dong2018strategic}, mechanism design under strategic manipulation \cite{chen2018strategyproof,eliaz2018model, kephart2015complexity}, and studying the manipulation costs that actually change the inherent qualifications \cite{kleinberg2018classifiers,miller2019strategic}. 
The focus of this paper is theoretically demonstrating the role of randomness and noise in the strategic setting. 

\subsection{Our contributions} 
Above, we talked about how strategic manipulation can deteriorate the classification accuracy and lead to unfair classification. We investigate the different scenarios of the classification task that help in regaining the lost accuracy and fairness guarantees. Our entire work is based on \emph{one-dimensional feature space}.

\subsubsection{Randomized classifiers} 
Firstly, we formulate the strategic classification task, when the published classifier is randomized. Instead of publishing a single binary classifier (for 2 classification outcomes, 0 and 1), the Jury publishes a distribution of classifiers and promises to pick the final classifier from that distribution. Another interpretation  is that the Jury assigns a value in $[0,1]$ to each feature value, which represents the probability of an individual with this feature being classified as $1$. The individuals manipulate their features, after knowing the set of classifiers but not the final classifier, incurring certain costs according to the \emph{cost function}. 

Not surprisingly, we show through examples that a probabilistic classifier can achieve strictly higher expected accuracy and efficiency than any binary classifier under strategic setting. Note that, without any strategic manipulation, a randomized classifier has no advantage over deterministic classifiers in terms of classification accuracy. The intuition is as follows: using randomness, the Jury can discourage the individuals from manipulating their features by making the advantage gained by any such a manipulation small enough.

For \emph{simple} cost functions, we then characterize the randomized classifier that achieves optimal efficiency. We prove that such a classifier sets the probabilities (of being classified as 1) such that none of the individuals have an incentive to manipulate their feature. Given two features $x$ and $x'$ in the feature space, let $c(x,x')$ denote the cost of manipulating one's feature from $x$ to $x'$. Informally, we say a cost function $c$ is \emph{simple} when all the costs are non-negative, the cost  to manipulate to a ``less" qualified feature is $0$, and the costs are sub-additive, that is, manipulating your feature $x$ directly to $x''$ is at least easier than first manipulating it to $x'$ and then to $x''$. The characterization theorem, stated informally, is as follows:
\begin{theorem*} [Informal statement of Theorem \ref{thm:charac}]
For simple cost functions, the most efficient randomized classifier is such that the best response of all the individuals is to reveal their true features.
\end{theorem*}
This characterization, in addition to being mathematically clean, allows us to infer the following: let $A$ and $B$ be two subpopulations (identical in terms of qualifications) such that the costs to manipulation are \emph{higher} for individuals in $A$ than in $B$, then the optimal efficiency obtained for the subpopulation $A$ is greater than that in $B$. 

\subsubsection{Obstacles to using a randomized classifier}Till now, we have argued the benefits of using a probabilistic classifier. However, the degree to which it is possible to use or commit to a randomized strategy varies depending on the setting. There are two main drivers impeding the implementation of the most efficient Stacklberg equilibrium.
Firstly, in many real-life classification settings, it might be unacceptable to use a probabilistic classifier, for example, due to legal restrictions (applicants with identical features must obtain identical outcomes). Secondly, for the more complicated scenario with multiple classifiers (such as college admissions), the effect of each Jury on the overall market is small, hence, diminishing the incentive to stick to a randomized strategy `for the benefit of the market as a whole'. Informally, the best response of a single Jury, when the other classifiers commit to using a randomized classifier, is not a randomized classifier. And even if we got the Juries to commit to randomization, the final probabilities of classification depends on the number of classifiers ($k$) and hence, the implementation of the most efficient randomized classifier needs coordination between the multiple classifiers. Analyzing the equilibria for multiple classifiers is beyond the scope of this paper but we illustrate the instability of randomized classifier as follows.
We show that unless Jury is able to commit to the published randomized classifier, such a classifier is not a stable solution to strategic classification. As mentioned above, randomization helps because of the following observation: if the difference between the probabilities, of being classified as $1$ at \emph{adjacent} features is small, the individuals have no incentive to manipulate their features.  
 But, once the Jury knows that no one changed their feature, her best response, then, is to use the classifier that achieves best accuracy given the \emph{true} features. 
 
 Formally, we show (Theorem \ref{thm:unstable}) that for any published randomized classifier that achieves strictly higher accuracy compared to any deterministic classifier under Stackelberg equilibrium, Jury has an opportunity to improve its utility and get strictly better accuracy using a classifier different from the published. 
 
The shortcomings of a randomized classifier can be redeemed by addition of noise to the features. 

\subsubsection{Addition of noise to the features}This brings us to our second scenario that uses noisy features for classification. Every individual has an associated private signal that identifies their qualification. The Jury sees a feature that is a noisy representation of this private signal. The individuals, after incurring certain cost, can effectively manipulate their private signal such that the features are a noisy representation of this updated private signal. Again, the assumption is that such a manipulation didn't change the true qualifications of an individual. We give a realistic example of such a manipulation in Section \ref{sec:noise}. We show, through an example of a cost function and a noise distribution, that in the strategic setting, using a deterministic classifier, the Jury achieves better accuracy when the features are noisy than any deterministic classifier in the noiseless case, that is, when Jury gets to see the private signal. This is counter-intuitive at first glance because under no strategic manipulation, noise can only decrease Jury's accuracy.

We also show examples where noisy features can help in achieving fairer outcomes across subpopulations. Let $A$ and $B$ be two subpopulations \emph{identical} in qualifications but having different (but not extremely different) costs of manipulation (and $|A|\le |B|$; $A$ is a minority). We show, through an example, that no matter whether the minority has higher or lower costs of manipulation than the majority, it is at a disadvantage when Jury publishes a single deterministic classifier to optimize its overall accuracy (noiseless strategic setting). Here, by disadvantage, we mean that the minority has lower classification accuracy than the majority. Next, we show that the addition of appropriate noise to the private signals, in the same example, can ensure that Jury's best response classifier is fair across subpopulations. This is not that surprising as making the features completely noisy also lead to same outcomes for the subpopulations. However, such an addition of noise can also sometimes increase Jury's overall accuracy (improving both accuracy and fairness). We consider the case where the Jury would publish a single classifier for both the subpopulations (for e.g., either because $A$ is a protected group and the Jury is not allowed to discriminate based on the subgroup membership or because the Jury has not yet identified these subpopulations and the differences in their cost functions).  Informally, our results, can be stated as follows:
\begin{theorem*}[Informal statement of Theorems \ref{thm:3},\ref{thm:4},\ref{thm:5}]
Let $A$ and $B$ be two subpopulations that are \emph{identical} in qualifications. Let $c_A\not =c_B$ be the cost functions for subpopulations $A$ and $B$ respectively. In Case 1, Jury gets to see the private signals and publishes a single deterministic classifier that achieves optimal overall accuracy (sum over the two subpopulations) under the Stackelberg equilibrium (for the cost functions $c_A$ and $c_B$). In Case 2, the features are noisy representations of the private signal; Jury publishes a single deterministic classifier that achieves optimal overall accuracy under the Stackelberg equilibrium (knowing that the features are noisy). There exists an instantiation of the ``identical qualifications" such that 
\begin{enumerate}
\item If $|A|<|B|$, that is, $A$ is a minority, for a wide set of costs functions $c_A,c_B$, $A$ is always at a disadvantage when in Case 1.  
\item There exists a setting of the ``noise" ($\eta$) for each of the above cost functions, such that, Jury's best response in Case 2, is always fair, that is, achieves equal classification accuracy on the subpopulations.
\item There exists cost functions $c_A, c_B$ from this wide set of cost functions, and corresponding noise $\eta$, such that Jury's accuracy in Case 2 is strictly better than in Case 1. 
\end{enumerate}
\end{theorem*}
This result has potentially interesting policy implications, since it is easier to commit to using noisier signals (for example by restricting the types of information available to the Jury) than to commit to disregarding pertinent information ex-post (as in randomized classification). Therefore, future mechanism design efforts involving strategic classification should carefully consider the mechanisms of information disclosure to the Jury.

\subsection{Related Work}
 \cite{hardt2016strategic,bruckner2011stackelberg} initiated the study of strategic classification through the lens of Stackelberg competition. \cite{hu2018disparate,milli2018social,immorlica2019access} study the effects of strategic classification on different subpopulations and how it can exacerbate the social inequity in the world. \cite{hu2018disparate} also made the observation that a single classifier would have varying classification accuracies across subpopulations with different costs of manipulation. \cite{milli2018social} defined a concept called ``social burden" of a classifier to be the sum of the minimum effort any qualified individual has to put in to be classified as $1$. Thus, the subpopulations with higher costs of manipulation would have worse social burden and might be at a disadvantage. In such situations, intuitively, one would think that subsidizing the costs for the disadvantaged population might help.  \cite{hu2018disparate} showed that cost subsidy for disadvantaged individuals can sometimes lead to worse outcomes for the disadvantaged group. 

In the present paper, we observe that the addition of noise, counter-intuitively, can help Jury's accuracy as well as serve the fairness concerns. There are many examples in game theory where loss of information helps an individual in strategic setting, for example, \cite{engelbrecht1986value}. \cite{kannan2019downstream,immorlica2019access} also studies the role of hiding information to serve fairness. \cite{frankel2019improving} has a brief discussion at the end of the paper on making manipulated data more informative through addition of noise to the features (this was put online a couple of months after the first version of our paper was made online). 

Another work related to Theorem \ref{thm:charac} of the present paper is \cite{kephart2016revelation}, which studies the scope of truthful mechanisms when the agents incur certain costs for misreporting their true type. In particular, the paper gives conditions, on the misreporting costs, that allow the revelation principle to hold, that is, any mechanism can be implemented by a truthful mechanism, where all the agents reveal their true types. The main difference between \cite{kephart2016revelation} and our paper is that the former allows the use of monetary transfers to the agents to develop truthful mechanisms and such transfers don't impact the objective value of the mechanism.

\subsection{Organization} 
We formalize the model used for strategic classification in Section \ref{sec:model}. In Section \ref{sec:random}, we show how randomness helps in achieving better accuracy and efficiency. We also characterize the classifiers that achieve optimal efficiency for \emph{simple} cost functions. In Section \ref{sec:equil}, we investigate the stability of randomized classifiers. In Section \ref{sec:noise}, we investigate the role of noisy features in strategic classification. 

\section{Preliminaries}\label{sec:model}
In this paper, we concern ourselves with classification based on a one-dimensional feature space $\X$. In many of the examples, our feature space $\X\subseteq \R$ is discrete, hence, we use sum ($\sum$) in many of the definitions, but, these definitions are well-defined when $\X$ is taken to be continuous (for e.g., $\R$) by replacing sum ($\sum$) with integrals ($\int$) and probability distributions with probability density functions. We use the notation $\N(z,\sigma)$ to denote the gaussian distribution with mean $z$ and standard deviation $\sigma$. We say a function $f:\X\rightarrow\{0,1\}$ is a threshold function (classifier) with threshold $\tau$ if
$$
f(x)=\begin{cases}
1 &\quad \text{if } x\ge \tau\\
0 &\quad \text{otherwise}
\end{cases}
$$
We also use $\one_{x\ge \tau}$ to denote a threshold function (classifier) with threshold $\tau$. Sometimes, we will use  $\one_{x> \tau}$ that classifies $x$ as 1 if and only if $x>\tau$.

\subsection{The Model}
Let $\X$ be the set of features. Let $\pi:\X\rightarrow[0,1]$ be the probability distribution over the feature set realized by the individuals. Let $h:\X\rightarrow [0,1]$ be the true probability of an individual being qualified (1) given the feature. We also refer to it as the true qualification function. Let $c(x,x')$ be the cost incurred by an individual to manipulate their feature from $x$ to $x'$ (We also use words, change and move, to refer to this manipulation). The classification is modeled as a sequential game where a Jury publishes a classifier (possibly probabilistic) $f:\X\rightarrow[0,1]$ and contestants (individuals) can change their features (after seeing $f$) as long as they are ready to incur the cost of change. The previous papers in the area considered the task of strategic classification when the published classifier is deterministic binary classifier. Here, we formalize the Stackelberg prediction game for probabilistic classifiers.

Given $f$, we define the best response of a contestant with feature $x$\footnote{Such a best response model has been studied in the literature, for example, \cite{wilkens2016mechanism}.}, as follows
\begin{equation}\label{eq:11}
\Delta_f(x)=\text{argmax}_{y\in ( \{x\}\cup\{x'\mid (f(x')-f(x))>c(x,x')\})}(f(y))
\end{equation}
We will denote it by $\Delta$ when $f$ is clear from the context. $\Delta(x)$ might not be well defined if there are multiple values of $y$ that attains the maximum. In those cases, $\Delta(x)$ is chosen to be the smallest $y$ amongst them. In words, you jump to another feature only if the cost of jumping is less than the advantage in being classified as 1.

We define the Jury's utility for publishing $f$ ($U(f)$) as the classification accuracy with respect to $h(x)$. Thus, Jury's utility for publishing $f$ is
\begin{align*}
U(f)&=\sum_{x\in\X} \pi(x) [f(\Delta(x))\cdot h(x)+(1-f(\Delta(x))\cdot (1-h(x))]\\
&=\sum_{x\in\X} \pi(x) [f(\Delta(x))\cdot (2h(x)-1)+1-h(x)]
\end{align*}

Consider the following example: let $\X$ be the set of numbers $\{1,...,100\}$. An individual is qualified if and only if their feature $x$ is at least 50, that is, $h(x)=1,\forall x\ge 50$ and 0 otherwise. When there is no threat of strategic manipulation of the features, Jury would publish $\one_{x\ge 50}$ as her classifier, which has perfect accuracy. 

Consider the strategic setting where the costs are as follows: 
$$ c(x,x')=
  \begin{cases}
   0      & \quad \text{if } x\ge x'\\
   \frac{x'-x}{3}  & \quad \text{otherwise}
  \end{cases}
$$
It's easy to see that if the Jury publishes $\one_{x\ge 52}$ as her classifier, she achieves perfect accuracy. However, due to this raised bar of acceptance, the individuals with feature $x\in\{50,51\}$ have to incur a cost to change their feature to $52$. Even though these individuals were fully qualified and identifiable as being qualified, the threat of strategic classification results in a situation where they have to put an extra effort to be classified as 1. We refer to this extra effort, put by the qualified individuals of a society, as the ``$\burden$". This quantity also becomes important when different subpopulations in a community have different cost structures resulting in different costs of strategies.\\

We define $C(f)=\sum_{x\in\X}\pi(x)[h(x)\cdot c(x,\Delta_f(x))]$ to be the $\burden$ for a published classifier $f$. Note that, we define the $\burden$ differently from the ``social burden" of strategic classification defined in the \cite{milli2018social} paper. They define the social burden for any qualified individual to be the minimum effort she has to put in to be classified as 1 even if she doesn't put that effort in the end, whereas, we define the $\burden$ for any qualified individual as the effort she actually puts-in in the classification game. We consider the ``$\burden$" to be of \emph{independent} interest. While higher costs of manipulation for a subpopulation leads to higher ``social burden" for the same subpopulation, it might not always lead to higher cost of strategy for that subpopulation.  Varying costs amongst subpopulations can lead to varied efforts put by different subpopulations of the world to achieve the same output and thus, is a source of fairness concern. \\

We define the efficiency of the classifier $f$ ($E(f)$)\footnote{We defined efficiency as $U(f)-C(f)$ for the simplicity of the presentation. Defining efficiency as $U(f)-\beta\cdot C(f)$ (for some $\beta>0$) doesn't effect the theorems except for Theorem \ref{thm:charac}, which is no longer true for $\beta<1$.} as follows: 
\begin{align*}
E(f)&=U(f)-C(f)\\
&=\sum_{x\in\X} \pi(x) [f(\Delta(x))\cdot h(x)+(1-f(\Delta(x))\cdot (1-h(x))]-\sum_{x\in\X}\pi(x)[h(x)\cdot c(x,\Delta(x))]\\
&=\sum_{x\in\X} \pi(x) [f(\Delta(x))\cdot h(x)+(1-f(\Delta(x))\cdot (1-h(x))- h(x)\cdot c(x,\Delta(x))]
\end{align*}

The focus of this paper is to demonstrate what role randomness and noise can play in strategic classification and not to give algorithms for learning the optimal or most efficient strategic classifier. We can present the ideas even by making the following assumptions on the cost function $c:\X\times\X\rightarrow\R$: 
\begin{enumerate}
\item $c(x,x')\ge 0, \; \forall x,x'\in\X$.
\item $c(x',x)=0,\;\forall x, x' \mid h(x')\ge h(x)$, that is, jumping to a lesser qualified feature is free.
\item $c(x,x'')\le c(x,x')+c(x',x''),\;\forall x,x',x'' \in\X$, that is, the costs are sub-additive. 
\item $c(x,x')\le c(x,x''),\;\forall x, x',x''\mid h(x'')\ge h(x')$, that is, jumping to a lesser qualified feature is easier.
\item $c(x',x'')\le c(x,x''),\;\forall x, x',x''\mid h(x')\ge h(x)$, that is, jumping from a lesser qualified feature is harder.
\end{enumerate}
The last two points are implied by the first three, we wrote them as separate points for completeness. We call the cost function \emph{simple} if it satisfies all the above assumptions. 

By the virtue of the definition of simple cost functions, without loss of generality, we assume that $h$ is monotonically increasing with the feature $x$, that is, 
\[\forall x,x'\in\X, \;\; x'\ge x\implies h(x')\ge h(x)\]

Next, we mention a special kind of cost function that satisfies the assumptions: 
$$c(x,x')=\max(a(x')-a(x),0)$$
 where the function $a:\X\rightarrow \R$ is monotonically increasing in $x$, that is, $x'\ge x\implies a(x')\ge a(x)$.

Given a cost function $c$, let \[\Lp{c}=\{f\mid f:\X\rightarrow[0,1],f(x')-f(x)\le c(x,x')\;\forall x,x'\in\X\}\]

Given the cost function $c$, we say $f$ satisfies the Lipschitz constraint if $f\in\Lp{c}$. Note that any classifier $f\in\Lp{c}$ is monotonically increasing with $x$, that is, $x'\ge x\implies f(x')\ge f(x)$. This is because $\forall x'\ge x, f(x)-f(x')\le c(x',x)=0$. And $\forall x\in\X, \Delta_f(x)=x$, that is, no one changes their feature if $f$ is the published classifier.\\

In Section \ref{sec:noise}, we generalize this model to the setting where the features are a noisy representation of an individual's private signal. An individual can make efforts to change their private signal but can't control the noise. The Jury only see the features and classifies an individual based on that. In Section \ref{sec:noise}, the fairness notion, we will concern ourselves with, is the classification accuracy of the published classifier across subpopulations.

\section{Committed Randomness Helps both Utility and Efficiency}\label{sec:random}
In this section, we compare the optimal utility and efficiency achieved by a deterministic binary classifier to a probabilistic classifier. Consider the following two scenarios:

\emph{Scenario 1}: The Jury commits to using a binary classifier $f:\X\rightarrow \{0,1\}$. The best response function $\Delta_f:\X\rightarrow \X$, Jury's utility from publishing $f$ ($U(f)$) and efficiency of the classifier $f$ ($E(f)$) are defined as in Section \ref{sec:model}.

\emph{Scenario 2}: The Jury publishes a probabilistic classifier $f:\X\rightarrow [0,1]$ and commits to it. The best response function $\Delta_f:\X\rightarrow \X$, Jury's utility from publishing $f$ ($U(f)$) and efficiency of the classifier $f$ ($E(f)$) are as defined in Section \ref{sec:model}.
Note that this is equivalent to when Jury publishes a list of deterministic classifiers and chooses a classifier uniformly at random from them. Contestants update their feature without knowing which classifier gets picked up at the end. \\

Before we delve into illustrating the examples where a randomized classifier achieves better utility and efficiency than any deterministic classifier, let's get an intuition of the differences in the two scenarios. 

In Scenario 1, without loss of generality, we can assume that the published classifier $f$ is a threshold classifier on the feature space \cite{hardt2016strategic}, that is, 
$$f(x)=
  \begin{cases}
    1      & \quad \text{if } x\ge \tau\\
    0  & \quad \text{otherwise}
  \end{cases}$$ for some $\tau\in\R$. This is because, if not, we can replace the published classifier $f$ with a threshold classifier, with $\tau=\min\{x\in\X\mid f(x)=1\}$ as the threshold. As it takes 0 cost to move to a lower feature, all the contestants with $x>\tau$ and $f(x)=0$ would have moved to $\tau$ even when the published classifier was $f$, resulting in the same effective classifier in the end. Therefore, the utility and efficiency of the published classifier remains the same after the replacement. An analog of a threshold classifier in the probabilistic setting would be a monotonically increasing probabilistic classifier ($f$ is a monotonically increasing classifier if $x\ge x'\implies f(x)\ge f(x')$). \\
  
Can we assume, without loss of generality, that the optimal classifier in terms of utility or efficiency is monotonically increasing with the features? The answer is no. The following example gives an illustration: Let $\X=\{1,2,3\}$ and each feature contains $\frac{1}{3}^{rd}$ of the population. Let 
$$h(x)=\begin{cases}
 1      & \quad \text{if } x\in\{2,3\}\\
0  & \quad \text{if } x=1\end{cases}$$
Let the cost of moving from 1 to 2 be 0 whereas the cost of moving from 1 or 2 to 3 be $0.9$. The classifier $f$ defined as 
$$f(x)=\begin{cases}
 1      & \quad \text{if } x=3\\
0  & \quad \text{if } x=2\\
0.1 &\quad \text{if } x=1\end{cases}$$ achieves almost perfect accuracy of $\sim 0.97$. We gave a lower probability of acceptance to the feature $2$ to incentivize them to change their feature to 3, while the contestants with feature $1$ had no incentive to change, when they were already being classified as 1 with probability $0.1$. 
For any monotone classifier $g$, as $g(1)\le g(2)$, all the contestants with feature 1 would definitely move to wherever the contestants with $x=2$ move to. Therefore, accuracy of any monotone classifier for this classification task is at most ~$0.66$ as the contestants at 1 and 2 cannot be classified differently. 
  
\emph{Remark}: This example is interesting for another reason, a deterministic classifier can never distinguish between features where the cost of changing into each other is 0, but a randomized classifier can.  \\
 
The above is also an example of when a randomized classifier achieves better accuracy (utility) than any deterministic classifier. Note that under no strategic manipulation, the threshold classifier with $\tau=\min\{x\mid h(x)\ge \frac{1}{2}\}$ as the threshold achieves the optimal accuracy. The ideal thing for a deterministic classifier to do would be to increase the threshold to $\tau'$ such that $c(\tau,\tau')=1$, as then, hopefully, only the contestants with feature $x\ge \tau$ would be incentivized to change their feature to $\tau'$. There are two cases when such a classifier potentially doesn't achieve the same accuracy as that of the one under no strategic manipulation: 1. When the costs to change between certain adjacent features are 0 (as seen by the above example). 2. When the threshold $\tau'$ doesn't exist either because of the discontinuities in the cost function or because the feature space is bounded and $\forall x\in \X, c(\tau,x)<1$.

The following example is of the second type and illustrates how randomization helps in getting strictly better utility and efficiency:

Let $\X=\{1,2\}$ and each feature contains half of the population. Let 
$$h(x)=\begin{cases}
 1      & \quad \text{if } x=2\\
0  & \quad \text{otherwise} 
\end{cases}$$
Let the cost of changing the feature from 1 to 2 be $0.5$. The the randomized classifier $f$ defined as follows:
$$f(x)=\begin{cases}
 1      & \quad \text{if } x=2\\
0.5  & \quad \text{if } x=1
\end{cases}$$ achieves an accuracy of $0.75$. The contestants at $x=2$ are happy as they are already being classified as 1 with probability 1. For the contestants at $x=1$, $f(2)-f(1)=0.5=c(1,2)$ and hence, they don't have an incentive to manipulate their feature. As all the contestant retain their true features, the efficiency of $f$ is also equal to $0.75$. As the feature space is bounded, there are only three options for a deterministic classifier: keep the threshold at 1 and classify everyone as 1; keep the threshold at 2 and you end up classifying everyone as 1, as the contestants at 1 change their feature to 2; classify everyone as 0. All these classifiers have 0.5 accuracy and at most 0.5 efficiency.

Consider the following simplistic example to illustrate the above point: say you want to classify whether an individual is intelligent based on the number of books they have in their study. As it gets easier to buy new books, the threshold, for the number of the books, to be classified as 1 is increased so as to keep achieving good accuracy. However, if the required threshold ``crosses" the upper bound on the number of books in a study, any deterministic classifier cannot achieve good accuracy and hence, the need to use randomization to de-incentivize the contestants to buy new books.  

In the mathematical example given above, the randomized classifier was set up such that none of the contestants had any incentive to change their feature. In the next subsection, we show that the most efficient classifier always looks like ``this" for ``simple" cost functions. That is, if the cost function $c$ satisfies the assumptions made in Section \ref{sec:model}, then for every true qualification function $h$, there exists a function $f_h\in \Lp{c}$ that achieves the optimal efficiency. 

\subsection{Characterization of the Most Efficient Classifier for Simple Cost Functions}
Recall, $E(f)=\sum_{x\in\X} \pi(x) [f(\Delta(x))\cdot h(x)+(1-f(\Delta(x))\cdot (1-h(x))- h(x)\cdot c(x,\Delta(x))]$.
Let $$E^*=\max_{f:\X\rightarrow[0,1]}\sum_{x\in\X} \pi(x) [f(\Delta(x))\cdot h(x)+(1-f(\Delta(x))\cdot (1-h(x))- h(x)\cdot c(x,\Delta(x))].$$
\begin{theorem} \label{thm:charac} For every monotone true qualification function $h:\X\rightarrow [0,1]$, probability distribution $\pi:\X\rightarrow [0,1]$ over the features, simple cost function $c$, there exists $g\in \Lp{c}$ such that $E(g)=E^*$.
\end{theorem}
In words, there exists a classifier in $\Lp{c}$ that maximizes efficiency (not necessarily unique). Note that for such a classifier, the $\burden$ is 0. 
\begin{proof}
Let $f$ be an efficiency maximizing classifier. We argue that $g:\X\rightarrow[0,1]$ defined as 
$$g(x)=\text{max}_y\{f(y)-c(x,y)\}$$ 
is in $\Lp{c}$ and satisfies $E(g)\ge E(f)$.\\
Let $\delta_f(x)=\text{argmax}_y\{f(y)-c(x,y)\}$. When $f$ is clear from the context, we will drop the subscript on $\delta$. Using definition of $\delta$, $g(x)\in[0,1]$ as $\forall x,y\in\X, \; f(y)-c(x,y)\le f(y)\le 1$ ($c(x,y)\ge 0$) and $\max_y\{f(y)-c(x,y\}\ge f(x)-c(x,x)\ge 0$. For all $x,x'\in \X$,
\begin{align*}
g(x')-g(x)&= f(\delta(x'))-c(x',\delta(x'))-f(\delta(x))+c(x,\delta(x))\\
&=f(\delta(x'))-c(x,\delta(x'))-f(\delta(x))+c(x,\delta(x))+\left(c(x,\delta(x'))-c(x',\delta(x'))\right) \\
&\le c(x,\delta(x'))-c(x',\delta(x'))\\
&\le c(x,x') \;\;\;\;\;(\text{sub-additivity})
\end{align*}
The first inequality follows the definition of $\delta$, that is, $\forall y\in\X, f(\delta(x))-c(x,\delta(x))\ge f(y)-c(x,y)$. Therefore, $f(\delta(x'))-c(x,\delta(x'))-f(\delta(x))+c(x,\delta(x))\le 0$. The second inequality follows from the fact that the cost function $c$ is simple and satisfies the sub-additivity condition.
This proves that $g\in\Lp{c}$. This implies, as observed previously, $\forall x\in\X, \Delta_g(x)=x$. Next, we show that $E(g)\ge E(f)$ and hence $E(g)=E^*$.
Efficiency of the classifier $g$ is 
\begin{align*}
E(g)&=\sum_{x\in\X} \pi(x) [g(\Delta_g(x))\cdot h(x)+(1-g(\Delta_g(x))\cdot (1-h(x))- h(x)\cdot c(x,\Delta_g(x))]\\
&=\sum_{x\in\X} \pi(x) [g(x)\cdot h(x)+(1-g(x)\cdot (1-h(x))- h(x)\cdot c(x,x)]\\
&=\sum_{x\in\X} \pi(x) [2\cdot g(x)\cdot h(x)-g(x)-h(x)+1]
\end{align*}
Efficiency of the classifier $f$ is 
\begin{align*}
E(f)&=\sum_{x\in\X} \pi(x) [f(\Delta_f(x))\cdot h(x)+(1-f(\Delta_f(x))\cdot (1-h(x))- h(x)\cdot c(x,\Delta_f(x))]\\
&=\sum_{x\in\X} \pi(x) [2f(\Delta(x))\cdot h(x)-f(\Delta(x))-h(x)+1- h(x)\cdot c(x,\Delta(x))]\\
\end{align*}
Therefore, 
\begin{align*}
E(g)-E(f)&
=\sum_{x\in\X} \pi(x) [2(g(x)-f(\Delta(x)))\cdot h(x)-(g(x)-f(\Delta(x)))+h(x)\cdot c(x,\Delta(x))]\\
&=\sum_{x\in\X} \pi(x) [(g(x)-f(\Delta(x)))\cdot (2h(x)-1)+h(x)\cdot c(x,\Delta(x))]
\end{align*}
\begin{claim}
$\forall x, \;\;[(g(x)-f(\Delta(x)))\cdot (2h(x)-1)+h(x)\cdot c(x,\Delta(x))]\ge 0$. 
\end{claim}
\proof
Recalling, $g(x)=f(\delta(x))-c(x,\delta(x))$. Using definition of $\delta$, we know that  
\begin{align}\label{eq:1}
g(x)=f(\delta(x))-c(x,\delta(x))\ge f(\Delta(x))-c(x,\Delta(x))
\end{align}
And, using definition of $\Delta$, we can show that 
\begin{align}\label{eq:2}
f(\Delta(x))\ge g(x)
\end{align}

This is because, either 
$f(\delta(x))-c(x,\delta(x))=f(x)$ and as $f(\Delta(x))\ge f(x)$, we get the inequality. Or, $f(\delta(x))-c(x,\delta(x))> f(x)$, which implies that $x$ has an incentive to change its feature to $\delta(x)$. Therefore, by the definition of $\Delta$, 
$f(\Delta(x))\ge f(\delta(x))\ge f(\delta(x))-c(x,\delta(x))$. The expression in the claim can be rewritten as
\begin{align*}
&(g(x)-f(\Delta(x)))\cdot (2h(x)-1)+h(x)\cdot c(x,\Delta(x))\\
&~~~~~~~~~=(g(x)-f(\Delta(x)))\cdot (h(x)-1)+h(x)\cdot (g(x)-f(\Delta(x))+c(x,\Delta(x)))\\
\end{align*}
As $g(x)-f(\Delta(x))\le 0$ from Equation \ref{eq:2} and $g(x)-f(\Delta(x))+c(x,\Delta(x))\ge 0$ from Equation \ref{eq:1}, the inequality follows from the fact that $0\le h(x)\le 1$. This proves the claim.\\

It's straightforward to see that $E(g)-E(f)\ge 0$ using the above claim. Therefore, we showed a classifier $g\in\Lp{c}$ such that $E(g)=E^*$.
\end{proof}
%
In words, \emph{when we are concerned with the efficiency of the published classifier, the optimal is achieved by a probabilistic classifier that has zero $\burden$ and gives individuals no incentive to change their feature.}

Ideally, the Jury would want to publish a classifier with $\tau=\min\{x\mid h(x)\ge \frac{1}{2}\}$ as the threshold. When restricted to choosing a classifier from $\Lp{c}$, the Jury is forced to increase the probability, of classifying a feature as 1, slowly around this threshold while maintaining the Lipschitz constraint. Smaller the costs of movement among features, slower is the allowed increase in the probabilities and hence, lower is the efficiency. We formalize this intuition as follows:

Let $c_1$ and $c_2$ be two simple cost functions such that $c_1$ \emph{dominates} $c_2$, that is, $\forall x,x'\in\X, c_2(x,x')\le c_1(x,x')$. After fixing a true qualification function $h$, and probability distribution $\pi$ over the features, we use $E^*_{c}$ to denote the optimal efficiency achieved by a published probabilistic classifier when the cost function is $c$. Given a classifier $f$, we use $E_{c}(f)$ to denote the efficiency of $f$ when the cost function is $c$.

\begin{proposition} Let $c_1$ and $c_2$ be two simple cost functions such that $c_1$ dominates $c_2$, then $E^*_{c_1}\ge E^*_{c_2}$. 
\end{proposition}
\begin{proof}
The proof follows easily after Theorem \ref{thm:charac}. There exists a classifier $g:\X\rightarrow[0,1]$ such that $g\in\Lp{c_2}$ and $E_{c_2}(g)=E^*_{c_2}$. The definition of ``$c_1$ dominates $c_2$" implies that $g\in\Lp{c_1}$. Let's calculate the efficiency of the classifier $g$ when the cost function is $c_1$. As $g\in\Lp{c_1}$, none of the contestants have any incentive to change their feature. Therefore, $E_{c_1}(g)=E_{c_2}(g)$ (the $\burden$ is 0 for both the cost functions). As $E^*_{c_1}$ is the optimal efficiency achieved under the cost function $c_1$, $E^*_{c_1}\ge E_{c_1}(g)=E_{c_2}(g)=E^*_{c_2}$.
\end{proof}

\emph{Remark}: Consider two subpopulations $A$ and $B$ with identical qualifications (same $h, \pi$) but varying costs of changing the features. If the costs of manipulation is always higher in $A$ than $B$, then, classification for $A$ is more efficient than classification for $B$. 
This would seem to be at odds with the observations of \cite{milli2018social} that the subpopulations with higher changing costs are at a disadvantage. 
We would like to point out that the classifiers that achieve optimal efficiency are different for different cost functions (subpopulations) whereas \cite{milli2018social} considers the setting of using a single classifier across subpopulations. 
Therefore, higher manipulation costs are better if we are allowed to use different classifiers for different subpopulations. Social burden as defined in \cite{milli2018social} can be high because it's costly for qualified contestants that are far below the ideal threshold $\tau$ to change their features to something above $\tau$. Another way to interpret this high social burden might be as the limitations of the qualification function $h$, that is, the limitations of the features to represent the actual qualifications, and maybe not as the fault of \emph{strategic} classification. This observation supports the definition of $\burden$. 

\section{Are Randomized Classifiers in Equilibrium from Jury's Perspective?}\label{sec:equil}
 As discussed in the Section \ref{sec:intro}, there are many obstacles to implementing a randomized classifier in the strategic setting. In this section, we illustrate the instability caused by the use of randomized classifiers (which becomes increasingly important while considering multiple classifiers). In Section \ref{sec:random}, we saw that a randomized classifier can achieve better accuracy and efficiency than any binary classifier. While maximizing efficiency, we further showed that the optimally efficient classifier is such that every contestant reveals their true feature. Once the Jury knows the contestants' true features, she can be greedy and classify the individuals using a threshold function with $\tau=\min\{x\mid h(x)\ge \frac{1}{2}\}$ as the threshold to achieve the best accuracy. Therefore, unless the Jury commits to using randomness, she has an incentive of not sticking to the promised randomized classifier. The question is: what's the best accuracy/efficiency achieved by a classifier that is in equilibrium even from Jury's perspective? We formalize this equilibrium concept as follows (the true qualification function $h$ and the cost function $c$ are fixed):
\begin{enumerate}
\item Jury publishes a randomized classifier $f:\X\rightarrow[0,1]$.
\item Contestants, knowing $f$, changes their feature from $x$ to $\Delta_f(x)$.
\item \label{it:3} $f$ is in equilibrium from Jury's perspective if given that the contestants changed their features according to the best response function $\Delta_f$, 
$f$ achieves the best classification accuracy, that is, for all classifiers $g\in\X\rightarrow [0,1]$,
\begin{align}
\label{eq:equil}
&\sum_{x\in\X} \pi(x) [f(\Delta_f(x))\cdot h(x)+(1-f(\Delta_f(x))\cdot (1-h(x))]\\
\nonumber
&~~~~~~~~~~~~~~~~-\sum_{x\in\X} \pi(x) [g(\Delta_f(x))\cdot h(x)+(1-g(\Delta_f(x))\cdot (1-h(x))]
\ge 0
\end{align}
\end{enumerate}

Using next theorem, we show that for any randomized classifier that is in equilibrium from Jury's perspective, there exists a binary classifier that achieves at least the same accuracy. 
\begin{theorem}\label{thm:unstable}
Given a monotone true qualification function $h$, probability distribution $\pi$ over the features, and a simple cost function $c$, let $f^*:\X\rightarrow\{0,1\}$ be the classifier that optimizes Jury's utility over the deterministic classifiers under Stackelberg equilibrium. Let $f:\X\rightarrow [0,1]$ be a randomized classifier such that $U(f)>U(f^*)$, then $f$ is not in an equilibrium from Jury's perspective (the notion defined above).
\end{theorem}

\begin{proof}
Equation \ref{eq:equil} implies that for all classifiers $g\in\X\rightarrow [0,1]$,
\begin{align*}
\sum_{x\in\X} \pi(x) [(f(\Delta_f(x))-g(\Delta_f(x)))\cdot (2h(x)-1)]
\ge 0
\implies\sum_{y\in\X} (f(y)-g(y))\cdot \sum_{x: \Delta_f(x)=y}\pi(x)(2h(x)-1)\ge 0
\end{align*}
Therefore, if $f$ is in equilibrium from the Jury's perspective, for all $y\in\X$ such that $f(y)\in(0,1)$, $\sum_{x: \Delta_f(x)=y}\pi(x)(2h(x)-1)=0$ otherwise Jury can choose $g(y)=1$ (or 0) depending on whether $\sum_{x: \Delta_f(x)=y}\pi(x)(2h(x)-1)>0$ (or $<0$) to increase her accuracy. Therefore, accuracy of the classifier $f$ is given by
\begin{align*}
U(f)&=\sum_{x\in\X} \pi(x) [f(\Delta_f(x))\cdot (2h(x)-1)+(1-h(x))]\\
&=\sum_{y\in\X} f(y)\cdot \sum_{x: \Delta_f(x)=y}\pi(x)(2h(x)-1)+\sum_{x\in\X}\pi(x)(1-h(x))\\
&=\sum_{y: f(y)=1} \sum_{x: \Delta_f(x)=y}\pi(x)(2h(x)-1)+\sum_{x\in\X}\pi(x)(1-h(x))
\end{align*}
Consider a binary classifier $f':\X\rightarrow \{0,1\}$ defined as follows: $f(x)\in[0,1)\implies f'(x)=0$ and $f(x)=1\implies f'(x)=1$. We can show that $U(f')\ge U(f)$. 
The contestants who change their features when $f'$ is the published classifier is a subset of $\{x\in\X \mid f(\Delta_f(x))\in(0,1]\}$ and as $\sum_{x :   f(\Delta_f(x))\in(0,1)}\pi(x)(2h(x)-1)=0$, the accuracy of $f'$ can only increase. This is because: $\forall x\in \X$ if $f(\Delta_f(x))=0$, then $f'(\Delta_{f'}(x))=0$ as otherwise if $x$ changed its feature under $f'$, it had an incentive to change under $f$ too. 

If $x'>x$, $f(\Delta_f(x')), f(\Delta_f(x))\in(0,1)$ and $x$ changes its feature under $f'$, then $x'$ has the incentive to change too as $c(x',x)=0$, and hence, the subset of $\{x\in\X \mid f(\Delta_f(x))\in(0,1)\}$ that change their features under $f'$ can only do a positive addition to the utility ($h$ is monotonically increasing with $x$ and $\sum_{x :   f(\Delta_f(x))\in(0,1)}\pi(x)(2h(x)-1)=0$). And, the contestants ($x$) who changed their features under $f$ such that $f(\Delta_f(x))=1$ would also change their features under $f'$ such that $f'(\Delta_{f'}(x))=1$ (as $f'(x)\le f(x)$) and are already included in the calculation of $U(f)$. 
\end{proof}

Disclaimer: $f'$ as defined above might also not be in equilibrium from Jury's perspective. The above theorem illustrates the following point: \emph{Jury doesn't benefit from randomized classifiers without creating instability in the system}. Unless Jury commits to the randomness, it has an incentive to use a classifier other than the promised, especially, in the situations where randomness helps in achieving better accuracy than any deterministic classifier in Stackelberq equilibrium. \\

 Can we somehow exploit this power of randomness in strategic classification while overcoming the obstacles to randomized classification?
The answer is yes -- make the features noisy. 

\section{Noisy Features Give the System Free Randomness}\label{sec:noise}
We formalize the setting with noisy features as follows: every individual has a private signal $y\in\X$. The true qualification function $h:\X\rightarrow[0,1]$ depends on $y$, that is, $h(y)$ is the probability of an individual being qualified (1) given that its private signal is $y$. 
Given a private signal $y$, a feature is drawn randomly from the distribution $p_y:\X\rightarrow[0,1]$, that is, $p_y(x)$ is the probability that an individual's feature is $x$ when their private signal is $y$. If $\X=\R$, the right intuition for $p_y$ is it being $\N(y,\sigma)$ where $\N(y,\sigma)$ is the gaussian distribution with mean $y$ and standard deviation $\sigma$. Let $\pi:\X\rightarrow [0,1]$ be the probability distribution over the private signals $y$ realized by the individuals. \\

Let $c(y,y')$ be the cost incurred by the contestant to change their private signal from $y$ to $y'$. The contestants can put effort to change their private signals but the feature would still be drawn randomly using the updated private signal (An example for such a modification would be as follows: SAT score is a random variable given the inherent intelligence. The randomness comes from various factors-- your preparation, sleep on the day of the test, etc. You can put more effort in preparing for the exam by cramming certain subjects. This might not change your inherent qualification but, as long as the SAT scores are concerned, would have the same effect of ``changing your private signal", that is, from the expected SAT score you achieve after cramming, it would seem that you were more intelligent to start with.) \\

The classification is again modeled as a sequential game where a Jury publishes a deterministic classifier $f:\X\rightarrow \{0,1\}$. We restricts ourselves to deterministic classifiers due to the observations made in Section \ref{sec:equil}. Contestants change their private signals as long as they are ready to incur the cost of change. Given a private signal $y$, let $q_f(y)$ denote the probability of a contestant, with private signal $y$, being classified as 1 when $f$ is the classifier. Therefore,
\begin{equation}\label{eq:51}
q_f(y)=\sum_{x\in\X}p_y(x)\cdot f(x)
\end{equation}
Given $f$, the best response of a contestant with private signal $y$ is given as,
\begin{equation}\label{eq:52}
\Delta_f(y)=\text{argmax}_{z\in \{y\}\cup\{y'\mid q_f(y')-q_f(y)>c(y,y')\}}(q_f(z))
\end{equation}

We will denote it by $\Delta$ when $f$ is clear from the context. $\Delta(y)$ might not be well defined if there are multiple values of $z$ that attains the maximum. In those cases, $\Delta(y)$ is chosen to be the smallest $z$ amongst them. In words, you jump to another private signal only if the cost of jumping is less than the advantage in being classified as 1. Even though $f$ is deterministic, due to noisy features, the effective classifier given the private signal $y$ ($q_f$) is probabilistic. Therefore, we will see below that the noise allows us similar advantages as that of a probabilistic classifier. \\

The accuracy of the classifier $f$ is defined as follows:
\begin{align*}
U(f)&=\sum_{y\in\X}\pi(y)[q_f(\Delta(y))\cdot h(y)+(1-q_f(\Delta(y))\cdot (1-h(y))]
\end{align*}

And the efficiency of the classifier $f$ is defined as follows:
\begin{align*}
E(f)&=\sum_{y\in\X}\pi(y)[q_f(\Delta(y))\cdot h(y)+(1-q_f(\Delta(y))\cdot (1-h(y))]-\sum_{y\in \X}\pi(y)[h(y)\cdot c(y,\Delta(y))]
\end{align*}
We assume that the function $h$ is monotonically increasing with $y$ and the cost function $c$ is simple. \\

\textbf{Noisy features helps both Jury's utility and efficiency}: Consider the following two cases: 1. Jury sees the private signal $y$ and bases her classifier on the private signal. 2. Jury sees a feature $x$ drawn from the distribution based on an individual's private signal and bases her classifier on the feature. \\

Under no strategic manipulation, it's clear that Jury would choose to be in case 1 and publish a threshold classifier with $\tau=\min\{y : h(y)\ge \frac{1}{2}\}$ as the threshold. Counterintuitively, under strategic classification, the Jury might choose to be in case 2 to achieve better accuracy and efficiency with a deterministic classifier. The following example illustrates this point. 

Let $\X=\{1,2\}$ and each private signal contains half of the population. Let 
$$h(y)=\begin{cases}
 1      & \quad \text{if } y=2\\
0  & \quad \text{otherwise} 
\end{cases}$$
Let the cost of changing the feature from 1 to 2 be $0.5$. Let the probability distribution of feature $x$
given the private signal $y$ be as follows:
\begin{align*}
&
p_2(x)=\begin{cases}
 1      & \quad \text{if } x=2\\
0      & \quad \text{if }x=1\\      
\end{cases}, &
p_1(x)=\begin{cases}   
 0.5  & \quad \text{if }  x=2\\     
0.5 & \quad \text{if } x=1
\end{cases}
\end{align*}

As seen in Section \ref{sec:random}, in case 1, Jury achieves 0.5 accuracy and 0.5 efficiency with an optimal deterministic classifier.

Whereas, in case 2, Jury can publish the following deterministic classifier and achieve an accuracy and efficiency of $0.75$: $f(x)=1$ if $x=2$ and 0 if $x=1$. If the private signal is $y=2$, then the Jury definitely sees a feature $x=2$ due to the definition of $p_y$. For a contestant with private signal $y=1$, their feature ends up being $x=2$ with probability $0.5 $ and hence, with the same probability, they are classified as 1. As the cost of changing the private signal from $1$ to $2$ is 0.5, its not worth the advantage of $0.5$ in the probability of being classified as 1. Therefore, all the contestants remain at their private signals and the published classifier achieves an accuracy and efficiency of $0.75$.

Till now, we saw how noisy features can help in achieving better accuracy and efficiency under strategic classification. In the following subsection, we will show that noisy features are also helpful for ensuring fairness. 

\subsection{Noisy Features achieve Fairer Equilibriums} 
Consider two subpopulations $A$ and $B$. For simplicity, these subpopulations are a partition of the individuals in the universe. Let $s_A$ denote the probability an individual from the universe is in subpopulation $A$. Similarly, $s_B$ ($s_A=1-s_B$). Let $h_A:A\rightarrow [0,1]$ be the true qualification function for the subpopulation $A$. Similarly, $h_B$. Let $c_A:\X\times \X\rightarrow \R$ be the cost function for the subpopulation $A$, that is, $c_A(y,y')$ is the cost of changing the private signal from $y$ to $y'$ for an individual in $A$. Similarly, $c_B$ is defined. Let $\pi_A:A\rightarrow [0,1]$ and $\pi_B$ be the probability distribution over the private signals realized by the subpopulations $A$ and $B$ respectively. 

Given a published deterministic classifier $f:\X\rightarrow\{0,1\}$, the best response of the contestant in subpopulation $A$ with private signal $y$ ($\Delta^A_f(y)$) is defined using $c_A$ as the cost function. Similarly, for subpopulation $B$, let $\Delta_f^B(y)$ denote the best response of the contestant in subpopulation $B$ with private signal $y$ and when the published classifier is $f$. We use $U_A(f)$ and $U_B(f)$ to denote the accuracy of the classifier $f$ on the respective subpopulations. \\

We consider the setting where $h_A=h_B=h$ and $\pi_A=\pi_B=\Pi$, but the cost functions $c_A$ and $c_B$ are different. In this section, we use the symbol $\Pi$ to denote the probability distribution over the private signals to avoid confusion with the Archimedes' constant $\pi$. 

In our first example, we show that even though the subpopulations are identical with respect to their qualifications, different costs can lead to unfair classification when classification is based on private signals. Through our second example, we show that the use of noisy features, for strategic classification, can lead to increase in the overall accuracy of classification as well as give fair classification. We evaluate the fairness of a classifier $f$ quantitively using the difference between the accuracies, that is, $|U_A(f)-U_B(f)|$.\\

Let's start with the example. $\X=\R$. Let the true qualification function for both the subpopulations be as follows:
$$
h(y)=\begin{cases}
 1      & \quad \text{if } y> d\\
\frac{y}{2d}+\frac{1}{2}     & \quad \text{if } y\in[-d,d]\\      
 0  & \quad \text{if } y<-d
\end{cases}
$$ 
where $d$ is a fixed large enough positive real number. Let the probability density function on the private signals realized by the subpopulations be as follows:
$$
\Pi(y)=\frac{e^{-\frac{y^2}{2t^2}}}{\sqrt{2\pi}t},
$$
that is, the gaussian distribution with mean 0 and standard deviation $t$. Again, $t$ is fixed positive real number. We assume
$d>>t$.

Let $\sigma_A$ and $\sigma_B$ be positive real numbers. The cost function for a subpopulation $S\in\{A,B\}$ is defined as follows:
 \begin{equation}\label{eq:54}
c_S(y,y')=\frac{(y'-y)^+}{\sqrt{2\pi}\sigma_S}
\end{equation}
Here, $(y'-y)^+=\max\{y'-y,0\}$. 

We start with the setting where the features are the private signals and not a noisy representation of them. 

\emph{Remark}: If the Jury is allowed to publish different classifiers for the two subpopulations, then she can achieve ``the best possible accuracy" on both the subpopulations. It's easy to see that the classifier $f_S:\X\rightarrow\{0,1\}$, defined as follows, achieves as much accuracy as a classifier under no strategic manipulation of the features can achieve on the subpopulation $S\in\{A,B\}$:
$
f_S(y)=\begin{cases}
1 &\quad\text{if } y\ge \sqrt{2\pi}\sigma_S\\
0 &\quad\text{otherwise}
\end{cases}
$.

All the contestants in a subpopulation $S$, with $0<y<\sqrt{2\pi}\sigma_S$ report their private signals to be $\sqrt{2\pi}\sigma_S$ as cost of this change is $<1$ whereas the advantage gained in the probability of being classified as 1 is 1. For all the contestants with private signal $y\le 0$, the cost of change is too high ($\ge 1$) and thus, they report their true private signals. Therefore, the classifier $f_S$ ends up classifying everyone with private signal $y>0$ as 1 which is the accuracy maximizing classification under the "no strategic manipulation" setting.

\textbf{How strategic classification leads to unfairness}: When $\sigma_A\not =\sigma_B$, the optimal classifiers for the subpopulations $A$ and $B$ are different and hence, when we choose a single classifier for both the subpopulations, we are bound to loose on the accuracy of at least one of the subpopulations. Through an example (Theorem \ref{thm:3}), we suggest that: \emph{while maximizing the overall accuracy over the universe, the minority group might be at a disadvantage irrespective of whether their costs to change the private signals are higher or lower than the majority subpopulation.} Without loss of generality, we assume that $A$ is the minority subpopulation, that is, $s_A\le s_B$. In many real life scenarios, the Jury would publish a single classifier for both the subpopulations either because $A$ is a protected group and the Jury is not allowed to discriminate based on the subgroup membership or because the Jury has not yet identified these subpopulations and the differences in their cost functions. 

\begin{theorem}\label{thm:3}
Let $A$ and $B$ be two subpopulations such that the true qualification functions, $h_A$, $h_B$, the probability density functions, $\pi_A$, $\pi_B$ and the cost functions $c_A$, $c_B$ are as instantiated above.

Assuming $|\sigma_A-\sigma_B|\le \frac{t}{\sqrt{2\pi}}$, let $f^*$ be the deterministic classifier that maximizes Jury's utility ($U(f)$), if $s_A<s_B$ and $\sigma_A\not =\sigma_B$ (the cost functions are different), then $U_A(f^*)<U_B(f^*)$, that is, the minority is at a disadvantage, even though their qualifications were identical ($h_A=h_B$, $\pi_A=\pi_B$).
\end{theorem}

\begin{proof}
Jury publishes a deterministic classifier and as there's no noise involved, without loss of generality, we can assume that $f$ is a threshold classifier on the space $\X$ (as $c_A$ and $c_B$ are simple cost functions). This assumption is justified in Section \ref{sec:random}. Given the classifier $f:\X\rightarrow\{0,1\}$ with threshold $\tau$, the best response of a contestant in the subpopulation $S\in\{A,B\}$ is given as follows:
$$
\Delta^S_f(y)=\begin{cases}
y &\quad\text{if } y\ge \tau\\
\tau &\quad\text{if } \tau- \sqrt{2\pi}\sigma_S<y<\tau\\
y &\quad\text{if } y\le \tau- \sqrt{2\pi}\sigma_S
\end{cases}
$$ 
   
The accuracy of the classifier $f$ for the subpopulation $S$ is given as follows:
\begin{align*}
U_S(f)
&=\int_{-\infty}^{\infty}\Pi(y)[f(\Delta^S(y))\cdot (2h(y)-1)+(1-h(y))]dy\\
\end{align*}
Let $c=\int_{-\infty}^{\infty}\Pi(y)[(1-h(y))]dy$ which is independent of the subpopulation and the classifier. 
Therefore, 
\begin{align*}
U_S(f)&=\left(\int_{-\infty}^{\infty}\Pi(y)[f(\Delta^S(y))\cdot (2h(y)-1)]dy\right)+c
\end{align*}

For the convenience of calculations, we will replace $h(y)$ with the following function 
$$
h'(y)=\frac{y}{2d}+\frac{1}{2}    
$$

As $d$ is large and $\Pi$ is a gaussian centered at 0, this change barely affects the utility values. To be precise, the difference in the 
utility calculations for any classifier $f$ while using $h'$ instead of $h$ is bounded by
\begin{align*}
\left|\int_{-\infty}^{\infty}\Pi(y)[f(\Delta^S(y))\cdot 2(h(y)-h'(y))]dy\right|\le &2\int_{-\infty}^{\infty}\Pi(y)[f(\Delta^S(y))|h(y)-h'(y)|]dy\\
\le&2\int_{-\infty}^{\infty}\Pi(y)\cdot |h(y)-h'(y)|dy\\
= &4\int_{d}^{\infty}\Pi(y)\cdot \left(\frac{y}{2d}-\frac{1}{2}\right) dy\\
\le&2\int_{d}^{\infty}\frac{e^{-\frac{y^2}{2t^2}}}{\sqrt{2\pi}t}\cdot \frac{y}{d} \;dy
~~=2\frac{te^{-\frac{d^2}{2t^2}}}{\sqrt{2\pi}d}
\end{align*}

As we take $d$ ($d>>t$) to be large enough, we would be able to ignore this difference. From now onwards, we use $h'$ as the ``true qualification function".

Therefore, the accuracy of the classifier $f$ over the subpopulation $S\in\{A,B\}$ can be approximated by
\begin{align*}
U_S(f)&=\left(\int_{-\infty}^{\infty}\Pi(y)[f(\Delta^S(y))\cdot (2h'(y)-1)]dy\right)+c
=\left(\int_{-\infty}^{\infty}\Pi(y)\cdot f(\Delta^S(y))\cdot \frac{y}{d}\;dy\right)+c\\
&=\left(\int_{\tau-\sqrt{2\pi}\sigma_S}^{\infty}\frac{e^{-\frac{y^2}{2t^2}}}{\sqrt{2\pi}t}\cdot \frac{y}{d}\;dy\right)+c
=\frac{t}{\sqrt{2\pi}d}e^{-(\tau-\sqrt{2\pi}\sigma_S)^2/2t^2}+c
\end{align*}
The second last equality follows from the definition of $\Delta_f^S$ and the fact that $f$ classifies everyone, with the updated private signal greater than or equal to $\tau$, as 1 and 0 otherwise.\\

The overall accuracy of the classifier $f$ is given by
\begin{align}
\label{eq:55}
U(f)=s_A\cdot U_A(f) +s_B\cdot U_B(f)=s_A\cdot \frac{t}{\sqrt{2\pi}d}e^{-(\tau-\sqrt{2\pi}\sigma_A)^2/2t^2}+s_B\cdot \frac{t}{\sqrt{2\pi}d}e^{-(\tau-\sqrt{2\pi}\sigma_B)^2/2t^2}+c
\end{align}

It's clear from the expression that the accuracy for the subpopulation $A$ is maximized at $\tau_A=\sqrt{2\pi}\sigma_A$ and that of $B$ is maximized at $\tau_B=\sqrt{2\pi}\sigma_B$. Consider the case when $s_A<s_B$. As $\tau_A\not =\tau_B$, and $U_B(f)$ has a larger weight in the expression, intuitively, while optimizing the overall accuracy, $\tau$ would try to achieve better accuracy for the subpopulation $B$, irrespective of whether $\sigma_A>\sigma_B$ or $\sigma_A<\sigma_B$, leading to unfairness across the subpopulations ($A$ being at a disadvantage).

It's complicated to calculate the optimal $\tau$, below we give a proof of the fact that the optimal $\tau$ would be such that $U_A(f)<U_B(f)$. To find the optimal value of $\tau$, we differentiate $U(f)$ with respect $\tau$ as follows:

\begin{align*}
\frac{dU(f)}{d\tau}&=s_A\cdot \frac{dU_A(f)}{d\tau} +s_B\cdot \frac{dU_B(f)}{d\tau}\\
&=-\frac{1}{\sqrt{2\pi}td}\left(s_A\cdot (\tau-\sqrt{2\pi}\sigma_A)\cdot e^{-(\tau-\sqrt{2\pi}\sigma_A)^2/2t^2}+s_B\cdot (\tau-\sqrt{2\pi}\sigma_B)\cdot e^{-(\tau-\sqrt{2\pi}\sigma_B)^2/2t^2}\right)
\end{align*}
Therefore, 
\begin{align*}
\frac{dU(f)}{d\tau}=0\implies &s_A\cdot (\tau-\sqrt{2\pi}\sigma_A)\cdot e^{-(\tau-\sqrt{2\pi}\sigma_A)^2/2t^2}+s_B\cdot (\tau-\sqrt{2\pi}\sigma_B)\cdot e^{-(\tau-\sqrt{2\pi}\sigma_B)^2/2t^2}=0\\
\implies  &\left|\frac{(\tau-\sqrt{2\pi}\sigma_A)\cdot e^{-(\tau-\sqrt{2\pi}\sigma_A)^2/2t^2}}{(\tau-\sqrt{2\pi}\sigma_B)\cdot e^{-(\tau-\sqrt{2\pi}\sigma_B)^2/2t^2}}\right|>1 ~~~~~(s_B>s_A)
\end{align*}

As $ze^{-\frac{z^2}{2t^2}}$ is maximized at $z=t$, as long as $|\sigma_A-\sigma_B|\le \frac{t}{\sqrt{2\pi}}$ (implying $|\tau-\sqrt{2\pi}\sigma_S|\le t$ for $S\in\{A,B\}$), the overall accuracy is maximized at a threshold $\tau$ such that $|\tau-\sqrt{2\pi}\sigma_A|>|\tau-\sqrt{2\pi}\sigma_B|$ and hence, $U_A(f^*)<U_B(f^*)$, where $f^*$ is the optimal classifier from Jury's perspective. The assumption, $|\sigma_A-\sigma_B|\le \frac{t}{\sqrt{2\pi}}$, can be interpreted as the subpopulations being different but not extremely different, which is reasonable assumption in many real life scenarios. 
\end{proof}

Next we show that, when the features are appropriately noisy, the optimal classifier from Jury's perspective is fair to the subpopulations. The intuition is as follows: if the noise is large enough such that none of contestants in either of the subpopulations want to manipulate their private signals, then the cost differences become irrelevant and hence, the optimal classifier achieves equal accuracy on both the subpopulations. You would think that this addition of noise would compromise Jury's utility. Subsequently, we show that adding noise might also improve the overall accuracy of the Jury's optimal classifier, therefore, addition of noise can make everyone happier. The latter is a continuation to the results at the start of Section \ref{sec:noise} about the usefulness of noise to the Jury under strategic classification.\\

\textbf{Noisy features lead to fairer outcomes}: Now, we analyze the setting with noisy features and prove the following theorem. The true qualification function $h$, cost functions ($c_A$ and $c_B$) and the probability density function $\Pi$ are as defined for the first example. Let $\sigma=\max\{\sigma_A,\sigma_B\}$. Given a private signal $y$, the features $x$ are distributed according to the gaussian with mean $y$ and standard deviation $\sigma$. The probability density function for the feature $x$ given the private signal $y$ is
$$
p_y(x)=\frac{e^{-\frac{(x-y)^2}{2\sigma^2}}}{\sqrt{2\pi}\sigma}.
$$

\begin{theorem} \label{thm:4}
Let $A$ and $B$ be two subpopulations such that the true qualification functions, $h_A$, $h_B$, the probability density functions, $\pi_A$, $\pi_B$ and the cost functions $c_A$, $c_B$ are as instantiated above.
When the features are drawn with a gaussian noise of mean 0 and standard deviation $\sigma$, such that, $\sigma\ge \sigma_A,\sigma_B$, if $f^*$ is the deterministic classifier that maximizes Jury's utility ($U(f)$), then $f^*$ is fair, that is, $U_A(f^*)=U_B(f^*)$. 
\end{theorem}
\begin{proof}

Again, we will replace the function $h$ with $h'$ (as in proof of Theorem \ref{thm:3}) while loosing an insignificant amount in all the calculations ($d>>t,\sigma$). Let $\Pi':\X\rightarrow[0,1]$ be the probability
density function over the features realized by each of the subpopulations. Let $H(x)$ ($H:\X\rightarrow[0,1]$) represent the probability of an individual being qualified (1) given that the Jury sees feature $x$. These functions are same for both the subpopulations. As the Jury only sees the feature and not the private signal, her accuracy is information-theoretically limited by these functions as we will describe below. Firstly,  $\Pi':\X\rightarrow[0,1]$ is given as follows:
\begin{align*}
\Pi'(x)&=\int_{-\infty}^{\infty}{\Pi(y)\cdot p_y(x) dy}=\int_{-\infty}^{\infty}{\frac{e^{-\frac{y^2}{2t^2}}}{\sqrt{2\pi}t}\cdot \frac{e^{-\frac{(x-y)^2}{2\sigma^2}}}{\sqrt{2\pi}\sigma} dy}\\
&=\int_{-\infty}^{\infty}\frac{e^{-\frac{x^2}{2(\sigma^2+t^2)}}}{\sqrt{2\pi}t}\cdot \frac{e^{-(y-\frac{xt^2}{\sigma^2+t^2})^2/(2\frac{\sigma^2t^2}{\sigma^2+t^2})}}{\sqrt{2\pi}\sigma} dy\\
&=\frac{e^{-\frac{x^2}{2(\sigma^2+t^2)}}}{\sqrt{2\pi}t\cdot \sqrt{2\pi}\sigma}\int_{-\infty}^{\infty} e^{-(y-\frac{xt^2}{\sigma^2+t^2})^2/(2\frac{\sigma^2t^2}{\sigma^2+t^2})}dy\\
&=\frac{e^{-\frac{x^2}{2(\sigma^2+t^2)}}}{\sqrt{2\pi}t\cdot \sqrt{2\pi}\sigma}\sqrt{2\pi\frac{\sigma^2t^2}{\sigma^2+t^2}}=\frac{e^{-\frac{x^2}{2(\sigma^2+t^2)}}}{\sqrt{2\pi(\sigma^2+t^2)}}
\end{align*}
Therefore, the probability density function over the features realized by the subpopulations, with $\N(0,\sigma)$ gaussian noise, is itself a gaussian with mean $0$ and $\sqrt{(\sigma^2+t^2)}$ standard deviation. 

The qualification function given the features, $H$, is given as follows:
\begin{align*}
H(x)=\frac{1}{\Pi'(x)}\int_{-\infty}^{\infty}{\Pi(y)\cdot p_y(x)\cdot h(y) dy}
\end{align*}
We replace $h$ with $h'$, thus replacing $H$ with $H'$ as defined below:
\begin{align*}
H'(x)&=\frac{1}{\Pi'(x)}\int_{-\infty}^{\infty}{\Pi(y)\cdot p_y(x)\cdot h'(y) dy}=\frac{1}{\Pi'(x)}\int_{-\infty}^{\infty}{\frac{e^{-\frac{y^2}{2t^2}}}{\sqrt{2\pi}t}\cdot \frac{e^{-\frac{(x-y)^2}{2\sigma^2}}}{\sqrt{2\pi}\sigma}\cdot (\frac{y}{2d}+\frac{1}{2}) dy}\\
&=\frac{1}{2}+\frac{1}{\Pi'(x)}\int_{-\infty}^{\infty}\frac{e^{-\frac{x^2}{2(\sigma^2+t^2)}}}{\sqrt{2\pi}t}\cdot \frac{e^{-(y-\frac{xt^2}{\sigma^2+t^2})^2/(2\frac{\sigma^2t^2}{\sigma^2+t^2})}}{\sqrt{2\pi}\sigma} \cdot \frac{y}{2d}\;\;dy\\
&=\frac{1}{2}+\frac{1}{2d\cdot\Pi'(x)}\frac{e^{-\frac{x^2}{2(\sigma^2+t^2)}}}{\sqrt{2\pi}t\cdot \sqrt{2\pi}\sigma}\int_{-\infty}^{\infty} e^{-(y-\frac{xt^2}{\sigma^2+t^2})^2/(2\frac{\sigma^2t^2}{\sigma^2+t^2})}\cdot y \;\;dy\\
&=\frac{1}{2}+\frac{1}{2d\cdot\Pi'(x)}\frac{e^{-\frac{x^2}{2(\sigma^2+t^2)}}}{\sqrt{2\pi}t\cdot \sqrt{2\pi}\sigma}\cdot \sqrt{2\pi\frac{\sigma^2t^2}{\sigma^2+t^2}}\cdot \frac{xt^2}{\sigma^2+t^2}\\
&=\frac{1}{2}+\frac{t^2}{\sigma^2+t^2}\frac{x}{2d}
\end{align*}

Therefore, when there's no strategic manipulation, Jury would classify any individual with feature $x>0$ as 1 and 0 otherwise. This is because, $H'(x)>\frac{1}{2}$ if and only if $x>0$ and the Jury would classify a feature as 1 if and only if, in expectation, the individuals with that feature are more likely to be qualified. This is true irrespective of whether an individual is from the subpopulation $A$ or $B$ because these subpopulations are identical in terms of qualifications, that is, $h_A=h_B=h$ and $\pi_A=\pi_B=\Pi$.

 We show that for the cost functions defined above, if Jury publishes $f=\one_{x>0}$, as the classifier, then none of the contestants in both the subpopulations $A$ and $B$ have an incentive to change their private signal (under $\N(0,\sigma)$ gaussian noise). Hence, the Jury gets the best possible accuracy from these features and the classification is fair. 
For a subpopulation $S\in\{A,B\}$, let $q^S_f(y)$ denote the probability of a contestant, with private signal $y$, being classified as 1 when $f$ is the classifier. Therefore,
\begin{align*}
q_f^S(y)=\int_{-\infty}^{\infty}f(x)\cdot p_y(x)dx=\int_{0}^{\infty}\frac{e^{-\frac{(x-y)^2}{2\sigma^2}}}{\sqrt{2\pi}\sigma}dx
\end{align*}

For a subpopulation $S\in\{A,B\}$, let's calculate the advantage that a contestant, with private signal $y$, gets by changing its signal to $y'$ ($y'>y$, otherwise $q_f^S(y')\le q_f^S(y)$ ):
\begin{align*}
q_f^S(y')-q_f^S(y)&=\int_{0}^{\infty}\frac{e^{-\frac{(x-y')^2}{2\sigma^2}}}{\sqrt{2\pi}\sigma}dx-\int_{0}^{\infty}\frac{e^{-\frac{(x-y)^2}{2\sigma^2}}}{\sqrt{2\pi}\sigma}dx
~~~~~=\int_{-y'}^{\infty}\frac{e^{-\frac{x^2}{2\sigma^2}}}{\sqrt{2\pi}\sigma}dx-\int_{-y}^{\infty}\frac{e^{-\frac{x^2}{2\sigma^2}}}{\sqrt{2\pi}\sigma}dx\\
&=\int_{-y'}^{-y}\frac{e^{-\frac{x^2}{2\sigma^2}}}{\sqrt{2\pi}\sigma}dx~~~~\le \int_{-y'}^{-y}\frac{1}{\sqrt{2\pi}\sigma}dx~~~=\frac{y'-y}{\sqrt{2\pi}\sigma}
\end{align*}
As $\sigma=\max\{\sigma_A,\sigma_B\}$ and recalling the definitions of the cost functions $c_A$ and $c_B$ (Equation \ref{eq:54}), we get that
\begin{align*}
&q_f^A(y')-q_f^A(y)\le c_A(y,y')
&\text{and }
&&q_f^B(y')-q_f^B(y)\le c_B(y,y')
\end{align*}
Therefore, none of the contestants in any of the subpopulations have an incentive to change their private signals. The accuracy of the classifier 
$f$ on the subpopulation $A$ is given as
\begin{align*}
U_A(f)&=\left(\int_{-\infty}^\infty \Pi(y)[q^A_f(\Delta^A_f(y))\cdot(2h(y)-1)] dy\right)+c=\left(\int_{-\infty}^\infty \Pi(y)\int_{0}^{\infty}\frac{e^{-\frac{(x-y)^2}{2\sigma^2}}}{\sqrt{2\pi}\sigma}dx\cdot(2h(y)-1)dy\right)+c\\
&=\left(\int_{0}^\infty \left(\int_{-\infty}^{\infty}\Pi(y)\frac{e^{-\frac{(x-y)^2}{2\sigma^2}}}{\sqrt{2\pi}\sigma}\cdot(2h(y)-1)dy\right)dx\right)+c
=\left(\int_{0}^\infty \Pi'(x)\cdot (2H(x)-1) dx\right)+c
\end{align*}
Replacing $H$ with $H'$ without loosing much in the approximation, we get that 
\begin{align*}
U_A(f)=\left(\int_{0}^\infty \frac{e^{-\frac{x^2}{2(\sigma^2+t^2)}}}{\sqrt{2\pi(\sigma^2+t^2)}}
 \cdot \frac{t^2}{\sigma^2+t^2}\frac{x}{d} dx\right)+c=\frac{t^2}{\sqrt{2\pi(\sigma^2+t^2)}\cdot d}+c
\end{align*}
The calculations for $U_B(f)$ are exactly the same and hence, $U(f)=U_B(f)=U_A(f)=\frac{t^2}{\sqrt{2\pi(\sigma^2+t^2)}\cdot d}+c$. 
\end{proof}

Theorem \ref{thm:4} would hold for when we are concerned with multiple subpopulations as long as $\sigma\ge \sigma_S$ for every relevant subpopulation $S$. In words, using noisy features we \emph{can} ensure that the best response of a Jury, maximizing her own utility, is fair to all the subpopulations that are identical in terms of qualifications but different in terms of the costs to manipulate the private signals, as long as the costs of manipulation for a subpopulation are not too small. 

\textbf{Noisy features can also improve Jury's utility}: Next, we show that further in some cases, \emph{the addition of noise to the features is not only beneficial for ensuring fairness but might also achieve better overall accuracy under strategic classification compared to when a noiseless signal is used.} 

Retaining the instantiations of $h_A$, $h_B$, $\pi_A$, $\pi_B$, $c_A$, $c_B$ and $\sigma$ as above, consider the following two scenarios: 1. Jury bases her classifier on the private signal $y$. 2. The features are drawn with a gaussian noise of mean 0 and standard deviation $\sigma$ and Jury bases her classifier on the features ($x$). 

Let $f_0^*$ and $f^*_\sigma$ be the optimal classifiers under strategic classification in the two scenarios respectively. Let $U(f_0^*)$ be the overall classification accuracy (Jury's utility) under Scenario 1 and $U(f_\sigma^*)$ be the overall classification accuracy (Jury's utility) under Scenario 2. We assume that the subpopulations are equally populated, that is, $s_A=s_B$ for simplicity of calculations in the next theorem.

\begin{theorem}\label{thm:5}
There exists qualification functions, $h_A$, $h_B$, the probability density functions over the private signals, $\pi_A$, $\pi_B$, the cost functions $c_A$, $c_B$ and $\sigma>0$ such that, $U(f_\sigma^*)>U(f_0^*)$, that is, the Jury gets better classification accuracy when the features are drawn with a gaussian noise of mean 0 and standard deviation $\sigma$. Here, the subpopulations have identical qualifications ($h_A=h_B$, $\pi_A=\pi_B$) but different cost functions. 
\end{theorem}
\begin{proof}
We retain the instantiations of $h_A$, $h_B$, $\pi_A$, $\pi_B$, $c_A$, $c_B$ and $\sigma$ as above. As seen above, in Scenario 2, $\one_{x>0}$ is the classifier that optimizes Jury's utility and hence,
  $U(f_\sigma^*)= \frac{t^2}{\sqrt{2\pi(\sigma^2+t^2)}\cdot d}+c$.
 Actually, it's approximately equal to this but the error is extremely small ($e^{-\Omega(d)}$, $d>>t, \sigma$). 
 
 In Scenario 1, the utility of any threshold classifier ($f$) with $\tau$ as the threshold is given by Equation \ref{eq:55} (without loss of generality, we can optimize over threshold classifiers). Therefore,
\begin{align*}
U(f)=s_A\cdot \frac{t}{\sqrt{2\pi}d}e^{-(\tau-\sqrt{2\pi}\sigma_A)^2/2t^2}+s_B\cdot \frac{t}{\sqrt{2\pi}d}e^{-(\tau-\sqrt{2\pi}\sigma_B)^2/2t^2}+c
\end{align*}  
When $s_A=s_B=\frac{1}{2}$ and we assume that $|\sigma_A-\sigma_B|\le \frac{t}{\sqrt{2\pi}}$, it's easy enough to see that the above expression is maximized at 
$\tau=\frac{\sqrt{2\pi}\sigma_A+\sqrt{2\pi}\sigma_B}{2}$. Therefore, the optimal classification accuracy in Scenario 1, is 
$$U(f_0^*)= \frac{t}{\sqrt{2\pi}d}e^{-(\frac{\sqrt{2\pi}\sigma_A-\sqrt{2\pi}\sigma_B}{2})^2/2t^2}+c$$

For $\sigma_B=\sigma$, $\sigma_A=0.1\sigma$, $t=0.9\sqrt{2\pi}\sigma$, $U(f_\sigma^*)>U(f_0^*)$.
\end{proof}
This theorem corroborates the idea that not only the subpopulations, but even the Jury might prefer noisy features. In the above example, for simplicity, we assumed $s_A=s_B$. Therefore, the optimal classifier was fair even in the noiseless setting. But a slight tweak in $s_A$ so that $s_A<s_B$ wouldn't change Jury's utility, in Scenario 1, by much and thus, would give an example where the noiseless setting has both unfairness and lesser overall classification accuracy.\\

In this paper, we study the interaction of noise with strategic classification through some simple examples, and leave the task of generalizing these results for future research.

\section{Discussion}
The problem of classification (and the strategic classification problem it entails) is of tremendous importance both practically (affecting pretty much every industry) and theoretically (with implications ranging from algorithms to policy and law). Therefore, clarifying the role randomness plays in this specific family of games is an important goal. Just as in games, randomness may lead to better solution in strategic classification. Moreover, in many important settings (such as college admissions in some jurisdictions), the classifier is required to be deterministic by law — which is not a handicap for algorithmic classification, but is a handicap for strategic one.
In addition, we proved that, in many natural cases, any randomized classifier (based on one-dimension) that achieves strictly better accuracy than the optimal deterministic one is not stable from the classifier’s standpoint, thus illustrating the difficulty of implementing a randomized classifier in a more complicated scenario with multiple classifiers (such as college admissions). This motivates the use of noisy features as a commitment device, which can improve both accuracy and fairness, and is also practically possible (for example by restricting the types of information available to the classifier).

\bibliographystyle{alpha}
\bibliography{biblio}

\newcommand{\etalchar}[1]{$^{#1}$}
\begin{thebibliography}{HMPW16}

\bibitem[ALB16]{akyol2016price}
Emrah Akyol, Cedric Langbort, and Tamer Basar.
\newblock Price of transparency in strategic machine learning.
\newblock {\em arXiv preprint arXiv:1610.08210}, 2016.

\bibitem[BS11]{bruckner2011stackelberg}
Michael Br{\"u}ckner and Tobias Scheffer.
\newblock Stackelberg games for adversarial prediction problems.
\newblock In {\em Proceedings of the 17th ACM SIGKDD international conference
  on Knowledge discovery and data mining}, pages 547--555. ACM, 2011.

\bibitem[CPPS18]{chen2018strategyproof}
Yiling Chen, Chara Podimata, Ariel~D Procaccia, and Nisarg Shah.
\newblock Strategyproof linear regression in high dimensions.
\newblock In {\em Proceedings of the 2018 ACM Conference on Economics and
  Computation}, pages 9--26. ACM, 2018.

\bibitem[DRS{\etalchar{+}}18]{dong2018strategic}
Jinshuo Dong, Aaron Roth, Zachary Schutzman, Bo~Waggoner, and Zhiwei~Steven Wu.
\newblock Strategic classification from revealed preferences.
\newblock In {\em Proceedings of the 2018 ACM Conference on Economics and
  Computation}, pages 55--70. ACM, 2018.

\bibitem[ES]{eliaz2018model}
Kfir Eliaz and Ran Spiegler.
\newblock The model selection curse.
\newblock {\em American Economic Review: Insights}.

\bibitem[EW86]{engelbrecht1986value}
Richard Engelbrecht-Wiggans.
\newblock On the value of private information in an auction: ignorance may be
  bliss.
\newblock {\em BEBR faculty working paper; no. 1242}, 1986.

\bibitem[FK19]{frankel2019improving}
Alex Frankel and Navin Kartik.
\newblock Improving information from manipulable data.
\newblock {\em arXiv preprint arXiv:1908.10330}, 2019.

\bibitem[HIV19]{hu2018disparate}
Lily Hu, Nicole Immorlica, and Jennifer~Wortman Vaughan.
\newblock The disparate effects of strategic manipulation.
\newblock In {\em Proceedings of the Conference on Fairness, Accountability,
  and Transparency}, pages 259--268. ACM, 2019.

\bibitem[HMPW16]{hardt2016strategic}
Moritz Hardt, Nimrod Megiddo, Christos Papadimitriou, and Mary Wootters.
\newblock Strategic classification.
\newblock In {\em Proceedings of the 2016 ACM conference on innovations in
  theoretical computer science}, pages 111--122. ACM, 2016.

\bibitem[ILZ19]{immorlica2019access}
Nicole Immorlica, Katrina Ligett, and Juba Ziani.
\newblock Access to population-level signaling as a source of inequality.
\newblock In {\em Proceedings of the Conference on Fairness, Accountability,
  and Transparency}, pages 249--258. ACM, 2019.

\bibitem[KC15]{kephart2015complexity}
Andrew Kephart and Vincent Conitzer.
\newblock Complexity of mechanism design with signaling costs.
\newblock In {\em Proceedings of the 2015 International Conference on
  Autonomous Agents and Multiagent Systems}, pages 357--365, 2015.

\bibitem[KC16]{kephart2016revelation}
Andrew Kephart and Vincent Conitzer.
\newblock The revelation principle for mechanism design with reporting costs.
\newblock In {\em Proceedings of the 2016 ACM Conference on Economics and
  Computation}, pages 85--102, 2016.

\bibitem[KR19]{kleinberg2018classifiers}
Jon Kleinberg and Manish Raghavan.
\newblock How do classifiers induce agents to invest effort strategically?
\newblock In {\em Proceedings of the 2019 ACM Conference on Economics and
  Computation}, pages 825--844. ACM, 2019.

\bibitem[KRZ19]{kannan2019downstream}
Sampath Kannan, Aaron Roth, and Juba Ziani.
\newblock Downstream effects of affirmative action.
\newblock In {\em Proceedings of the Conference on Fairness, Accountability,
  and Transparency}, pages 240--248. ACM, 2019.

\bibitem[MMDH19]{milli2018social}
Smitha Milli, John Miller, Anca~D Dragan, and Moritz Hardt.
\newblock The social cost of strategic classification.
\newblock In {\em Proceedings of the Conference on Fairness, Accountability,
  and Transparency}, pages 230--239. ACM, 2019.

\bibitem[MMH19]{miller2019strategic}
John Miller, Smitha Milli, and Moritz Hardt.
\newblock Strategic adaptation to classifiers: A causal perspective.
\newblock {\em arXiv preprint arXiv:1910.10362}, 2019.

\bibitem[WCNT16]{wilkens2016mechanism}
Christopher~A Wilkens, Ruggiero Cavallo, Rad Niazadeh, and Samuel Taggart.
\newblock Mechanism design for value maximizers.
\newblock {\em arXiv preprint arXiv:1607.04362}, 2016.

\end{thebibliography}
\end{document}